\documentclass{article}
\usepackage{ijcai24}

\usepackage{times}
\usepackage{appendix} 
\usepackage{soul}
\usepackage{comment}
\usepackage{url}
\usepackage[hidelinks]{hyperref}
\usepackage[utf8]{inputenc}
\usepackage[small]{caption}
\usepackage{graphicx}
\usepackage{multirow}
\usepackage{makecell}
\usepackage{xcolor}
\usepackage{subfigure}
\usepackage{subcaption}
\usepackage{amsmath}
\usepackage{mathtools}
\usepackage{amssymb}
\usepackage{amsthm}
\usepackage{booktabs}
\usepackage{algorithm}
\usepackage{algorithmic}
\usepackage[switch]{lineno}

\newtheorem{prop}{Proposition}

\title{CASA: CNN Autoencoder-based Score Attention for\\ Efficient Multivariate Long-term Time-series Forecasting }

\author{
Minhyuk Lee$^{1 \; *}$
\and
HyeKyung Yoon$^{3 \; *}$\And
MyungJoo Kang$^{1, 2, 3 \; \dagger}$\\
\affiliations
$^1$Department of Mathematical Sciences, $^2$Research Institute of Mathematics and\\ $^3$Interdisciplinary Program in Artificial Intelligence, Seoul National University\\
\emails
\{356min, yhk04150, mkang\}@snu.ac.kr,
}
\renewcommand\thesubfigure{(\alph{subfigure})} % 소문자 알파벳을 괄호로 감쌈

\begin{document}
%\begin{comment}
\maketitle
\footnotetext{$^*$ Equal contribution.}
\footnotetext{$^\dagger$ Corresponding author.}

\begin{abstract}
    Multivariate long-term time series forecasting is critical for applications such as weather prediction, and traffic analysis. In addition, the implementation of Transformer variants has improved prediction accuracy. Following these variants, different input data process approaches also enhanced the field, such as tokenization techniques including point-wise, channel-wise, and patch-wise tokenization. However, previous studies still have limitations in time complexity, computational resources, and cross-dimensional interactions. To address these limitations, we introduce a novel CNN Autoencoder-based Score Attention mechanism (CASA), which can be introduced in diverse Transformers model-agnosticically by reducing memory and leading to improvement in model performance. Experiments on eight real-world datasets validate that CASA decreases computational resources by up to 77.7\%, accelerates inference by 44.0\%, and achieves state-of-the-art performance, ranking first in 87.5\% of evaluated metrics. Our code is available at \url{https://github.com/lmh9507/CASA}.

\end{abstract}

\section{Introduction}

\begin{figure}[h!]
    \centering
    \includegraphics[width=0.85\columnwidth]{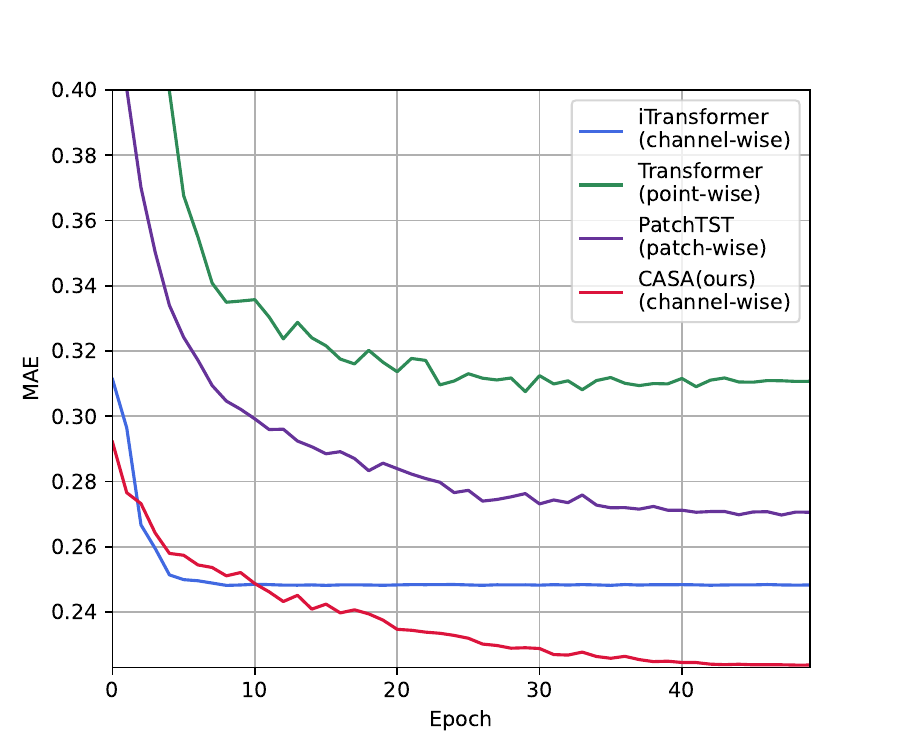} % 한 컬럼 너비로 설정
    \caption{The validation loss of iTransformer, Transformer, PatchTST, and our model on the Traffic dataset is evaluated. Point-wise and patch-wise implemented models exhibit lower performance compared to channel-wise models. However, while the iTransformer model rapidly saturates, CASA demonstrates consistent learning and achieves the lowest loss value.}
    \label{fig:val}
\end{figure}

Multivariate Long-Term Time Series Forecasting (LTSF) plays a pivotal role in real-world applications, including weather prediction, traffic flow analysis \cite{ji2023spatio}, and solar energy forecasting \cite{lai2018modeling}. LTSF has seen rapid advancements driven by the emergence of Transformer model \cite{vaswani2017attention}. Subsequent Transformer variants have further demonstrated the effectiveness of the multi-head self-attention mechanism in capturing temporal dependencies and cross-dimensional correlations. \cite{zhou2021informer,wu2021autoformer,liu2022pyraformer,zhou2022fedformer,liu2022non,zhang2023crossformer,nie2022time}.

Despite numerous efforts, Transformer-based models have not consistently outperformed CNN- or MLP-based architectures in the LTSF domain \cite{wu2022timesnet,ekambaram2023tsmixer,das2023long,zeng2023transformers}. Notably, DLinear \cite{zeng2023transformers}, a model constructed with simple linear layers, raises critical questions about the effectiveness and necessity of Transformer-based architectures in this field, especially considering their demanding computational and time resources. To address the aforementioned challenges, diverse tokenization techniques have been introduced into the core architecture of Transformer-family models \cite{liu2023itransformer,nie2022time}. We evaluate the effectiveness of three models, each employing a distinct tokenization technique, as illustrated in Figure \ref{fig:token}. Figure \ref{fig:val} demonstrates that channel-wise tokenization achieves the best performance among the three, as highlighted in \cite{yu2024revitalizing}. However, it still leads to rapid saturation during training. Moreover, computational cost and memory usage remain significant challenges. Our objective is to develop a more efficient model that delivers superior predictive performance while mitigating these drawbacks. \textbf{We tackle these issues by refining the self-attention mechanism, the cornerstone of Transformer-based models.}

\begin{figure*}[h!]
    \centering
    \includegraphics[width=0.8\textwidth]{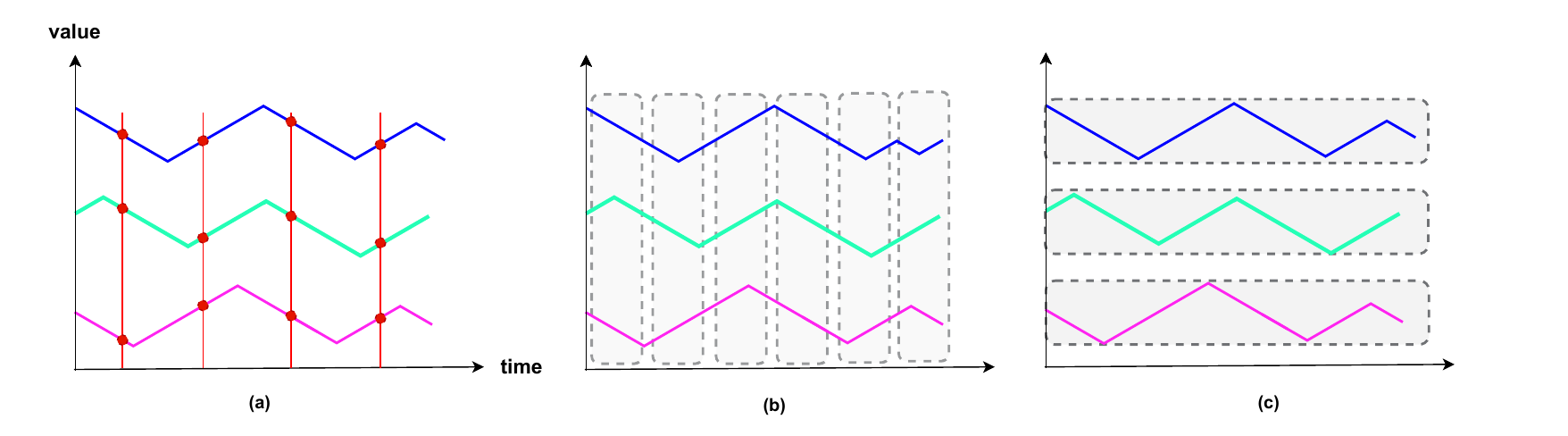} % 한 컬럼 너비로 설정
    \caption{(a) point-wise token (b) patch-wise token (c) channel-wise token}
    \label{fig:token}
\end{figure*}

In this paper, we propose \textbf{CNN Autoencoder-based Score Attention (CASA)}, a simple yet novel module that addresses the aforementioned gap by effectively capturing correlations and avoiding saturation, thereby facilitating consistent learning, designed to serve as an alternative to the conventional self-attention mechanism. We retain the vanilla Transformer encoder while substituting the attention mechanism with a CNN-based module. CASA approximates $\frac{Q K^T}{\sqrt{d_k}}$ rather than directly calculating it, as done in traditional multi-head self-attention. This design addresses a critical limitation of conventional methods by sufficiently accounting for significant correlation between variates in the calculation of attention scores (see Section 3.3 for details).

Our key contributions are as follows:
\begin{itemize}
    \item \textbf{We present a simple yet effective CNN Autoencoder-based Score Attention (CASA) module} as an alternative to self-attention. It scales linearly with the number of variates, input length, and prediction length. Compared to Transformer-based variants, CASA reduces memory usage by up to \textbf{77.7\%} and improves computational speed by \textbf{44.0\%}, depending on the dataset.
    \item \textbf{CASA can be agnostically integrated into Transformer models, regardless of the tokenization technique used} (e.g., point-wise, patch-wise, or channel-wise). To the best of our knowledge, this is the first module validated across individual tokenization techniques, successfully enhancing cross-dimensional information capture and improving prediction performance.
 \item CASA achieved \textbf{first place in 54 out of 64 metrics} across 8 real-world datasets and ranked highest in \textbf{14 out of 16 average metrics}, establishing itself as a highly competitive solution for multivariate LTSF.
\end{itemize}

\begin{figure*}[h!]
    \centering
    \includegraphics[width=0.9\textwidth]{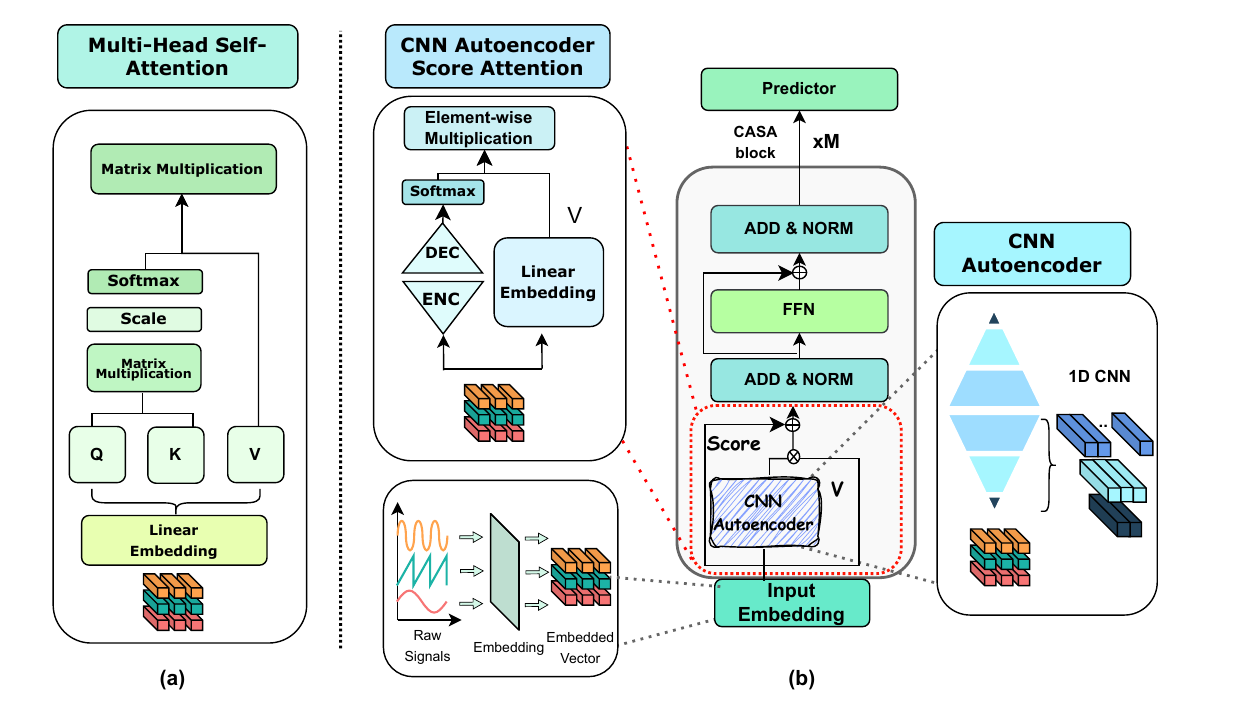} % 한 컬럼 너비로 설정
    \caption{(a) Conventional Self-Attention. (b) Overall architecture of our CASA block. The time-series data is embedded using channel-wise tokenization. The 1D CNN Autoencoder is then used to compute cross-dimensional information. The softmax output and the value are multiplied element-wise. Our CASA places a strong emphasis on capturing essential cross-dimensional information by calculating high-dimensional spatial relationships before compressing the channel information.}
    \label{fig:overall}
\end{figure*}
\vspace{-5pt}

\section{Related Works}
\paragraph{Transformer variants} The vanilla Transformer model \cite{vaswani2017attention}, widely recognized for its success in natural language processing, has also achieved notable advancements in time-series forecasting. Diverse Transformer variants have been introduced to enhance forecasting performance, which can be broadly grouped into three approaches. The first approach modifies the traditional self-attention mechanism with alternatives by incorporating specialized modules,  or pyramidal attention \cite{liu2022pyraformer}, to reduce memory requirements while capturing multi-resolution representations. Additional modifications, including the trapezoidal architecture \cite{zhou2021informer} and de-stationary attention \cite{liu2022non}, aim to improve robustness and address issues like over-stationarization. The second approach leverages frequency-domain techniques, such as Fast Fourier Transform (FFT) \cite{zhou2022fedformer} and auto-correlation mechanisms \cite{wu2021autoformer}, to better extract temporal patterns. The third approach introduces hierarchical encoder-decoder frameworks \cite{zhang2023crossformer} with routing mechanisms to capture cross-dimensional information, although these methods sometimes encounter challenges such as slower learning and higher computational demands.

\vspace{-5pt}

\paragraph{Alternatives of Transformers} While Transformer variants have significantly advanced the time-series forecasting domain, CNN-based models present promising alternatives. These approaches include methods that model segmented signal interactions \cite{scinet} and those that reshape 1D time-series data into 2D tensors \cite{wu2022timesnet}, enabling the capture of both inter-period and intra-period dynamics. Similarly, linear models \cite{zeng2023transformers} have demonstrated simplicity while achieving high prediction performance. However, these methods generally fall short of explicitly addressing cross-dimensional interactions, which are crucial for improving multivariate time-series forecasting.
Other methods have been developed to modify aspects of the Transformer architecture, particularly focusing on tokenization techniques. For instance, PatchTST \cite{nie2022time} segments input data into patches to extract local information within individual variates, while iTransformer \cite{liu2023itransformer} treats each variate as a token, enabling the self-attention mechanism to capture multivariate correlations. However, a common drawback of these methods is their reliance on self-attention, which demands substantial computational resources. Furthermore, their prediction performance remains suboptimal, with both models struggling to effectively capture multivariate dependencies, which critically impacts their predictive accuracy. In contrast, our proposed model, CASA, addresses these challenges by significantly reducing resource consumption while achieving superior prediction performance, offering an efficient and effective alternative to traditional self-attention-based approaches.

\section{Method}
This section provides an overview of CASA, including the motivation behind its architecture, a detailed explanation of its structure, and a complexity analysis. Section 3.1, we formulate the problem we aim to solve and define the necessary notations. Section 3.2 provides an explanation of the overall architecture. Section 3.3 discusses the limitations of conventional Transformers in capturing cross-dimensional interactions. Section 3.4 introduces CNN Autoencoder-based Score Attention (CASA), an improved self-attention module designed to address these issues. We confirm theoretical efficiency of CASA through complexity analysis.
\subsection{Problem Formulation}
In multivariate LTSF, given an input length $L$, the number of variates $N$, the number of layers $M$, and a prediction length $H$, denote the input and the output $X \in \mathbb{R}^{N \times L}$ and $Y \in \mathbb{R}^{N \times H}$ respectively. The hidden dimension is denoted as $D$, and the intermediate feature after embedding is represented as $Z_i \in \mathbb{R}^{N \times D} \; ( i \in \{0, 1, \cdots M\} )$. Since this is not a univariate time-series forecasting problem, the input is consistently represented as a matrix.

\begin{table*}[h]
\centering
\setlength{\tabcolsep}{3.5pt}
\small
\begin{tabular}{lccccc}
\toprule
 & \textbf{CASA} & \textbf{SOFTS} & \textbf{iTransformer} & \textbf{PatchTST} & \textbf{Transformer} \\ 
\midrule
\textbf{Complexity} & $O(NL + NH)$ & $O(NL + NH)$ & $O(N^2 + NL + NH)$ & $O(NL^2 + NH)$ & $O(NL + L^2 + HL + NH)$ \\ 
\midrule
\textbf{Memory}(MB) & 1684 & 1720 & 10360 & 21346 & 4772 \\ 
\midrule
\textbf{Inference}(s/iter) & 0.05 & 0.042 & 0.162 & 0.441 & 0.087 \\ 
\bottomrule
\end{tabular}
\caption{A complexity comparison conducted on the Traffic dataset among baseline models, including the vanilla Transformer, PatchTST, iTransformer, and SOFTS, with respect to window length L, number of channels N, and forecasting horizon H. Notably, the complexity of CASA scales linearly with N, L, and H. Detailed implementation information is provided in Appendix D.
}
\label{tab:complexity_comparison}
\end{table*}

\subsection{Overall Architecture}
The overall framework is depicted in Figure \ref{fig:overall}. Following prior works \cite{liu2023itransformer,yu2024revitalizing}, we adopt a channel-wise tokenization approach, utilizing the vanilla Transformer encoder as the backbone. To address the challenges outlined in Section 3.3, we enhance the Transformer by replacing only the attention mechanism with a 1D \textbf{CNN Autoencoder Score Attention (CASA)}. Although this modification alters only a very small portion of the overall model, it effectively reduces computational cost and memory usage while improving prediction performance.

The input $X$ is linearly embedded to produce the intermediate feature $Z_0$. The final feature $Z_M $, obtained by passing $Z_0$ through $M$ CASA blocks, is then fed into the predictor where the output becomes $Y$. The pipeline is summarized by the following equation.
\begin{eqnarray}
    Z_0 = \mathrm{Embedding}(X) \\
    Z_{i+1} = \mathrm{CASA\;block}(Z_i) \\
    Y = \mathrm{Predictor}(Z_M)
\end{eqnarray}

\subsection{Limitation of Self-Attention} 
We demonstrate that the existing self-attention mechanism does not sufficiently consider cross-dimension information when embedding queries and keys. the structure of the self-attention mechanism in the conventional Transformer is as follows ($f_j$: affine map):
\begin{equation}
    \mathrm{Attention}(Z_{i+1}) = \mathrm{softmax}\left( \frac{Q_{i+1}K_{i+1}^T}{\sqrt{d_k}} \right) * V_{i+1}
\end{equation}
\begin{equation}
    Q_{i+1} = f_1(Z_i),\;\, K_{i+1} = f_2(Z_i),\;\, V_{i+1} = f_3(Z_i)
\end{equation}
At this point, since $f_j$ is an affine map, queries and keys are computed through the following operations:
\begin{eqnarray}
    Q_{i+1} = Z_{i, 1}W_{i, 1} + b_{i, 1}, \quad K_{i+1} = Z_{i, 2}W_{i, 2} + b_{i, 2}
\end{eqnarray}
\begin{equation}
    \mathrm{where} \; W_{i, j} \in \mathbb{R}^{D \times D} \; \mathrm{and} \; b_{i, j} \in \mathbb{R}^{N \times D}
\end{equation}

\begin{prop}\label{prop:qk}
Query and key embeddings are variate-independent operations in the conventional Transformer using channel-wise tokenization.
\end{prop}
\begin{proof}
See Appendix E.1.
\end{proof}

\begin{prop}\label{prop:qk2}
Query and key embeddings are time-independent operations in the conventional Transformer using point-wise and patch-wise tokenization.
\end{prop}
\begin{proof}
See Appendix E.2.
\end{proof}

 Proposition 1 and Proposition 2 imply that each tokenization method does not consider the correlation between variates or time points when embedding the key and query. Since multivariate time series exhibit correlations both between variates and across time points, this reduces the potential of the Transformer architecture. Especially, based on the Proposition \ref{prop:qk}, the Transformer using channel-wise tokenization does not directly incorporate cross-dimensional information when embedding the $r$-th variate into queries and keys. In other words, the self-attention mechanism embeds queries and keys through variate-independent feature refinement operations and then computes the attention map using $\frac{Q_i K_i^T}{\sqrt{d_k}}$. \textbf{In the LTSF domain, tokens (i.e., variates) exhibit inherent correlations} (see Section \ref{sec:dist}), \textbf{which reduce the effectiveness of variate-independent operations during feature refinement.} This limitation can hinder performance in the multivariate LTSF domain, where capturing correlations between variates is important.

 \subsection{CNN Autoencoder-based Score Attention}\label{sec:theoritical proof}
 To address the issue posed by the self-attention mechanism’s variate-independent operation, we treat each variate as a channel and apply a convolution instead of using the affine map from the conventional self-attention mechanism. This approach ensures that the operation becomes variate-dependent. In more detail, instead of directly computing $\frac{Q_i K_i^T}{\sqrt{d_k}}$, we designed a score network $\texttt{Score}$ to approximate this operation using 1D CNN Autoencoder. The modified structure of self-attention is as follows ($f$: affine map, $\circledast$: element-wise product):
\begin{eqnarray}
    \mathrm{Attention}(Z_{i+1}) = \mathrm{softmax}(\texttt{Score}(Z_i)) \circledast V_{i+1}
\end{eqnarray}
\begin{equation}
    V_{i+1} = f(Z_i)
\end{equation}
By approximating $\frac{Q_i K_i^T}{\sqrt{d_k}}$ via a CNN architecture instead of direct computation, we reduce complexity compared to the affine map (see the paragraph below), addressing the limitations of self-attention. This facilitates the development of a linear complexity model with enhanced performance (Section 4.1). To explain the score network in more detail, we adopted an inverted bottleneck autoencoder structure, inspired by previous research \cite{wilson2016deep,bengio2013representation}, which demonstrated that embedding low-dimensional features into a high-dimensional latent space can improve expressiveness. In summary, we leverage CNN operations to incorporate information across all variates, embedding them into a high-dimensional feature space before compressing the channels to retain only essential cross-variable information. Consequently, despite its simple architecture, the proposed module outperforms existing self-attention mechanisms while maintaining efficiency, constituting a significant contribution.

\begin{table*}[ht]
    \centering
    \scriptsize  % \uadf8\ub9bc \ud06c\uae30\ub97c \ub354 \uc791\uac8c \uc870\uc815
    \setlength{\tabcolsep}{5pt}  % \uc5f4 \uc0ac\uc774\uc758 \uc5ec\ubc29 \uc904\uc774\uae30
    \begin{tabular}{c c c c c c c c c c}
        \toprule
        \multirow{1}{*}{Dataset} &
        \multicolumn{1}{c}{\textbf{CASA}(ours)} & \multicolumn{1}{c}{\makecell{\textbf{SOFTS} \\ Han \textit{et al.}, \\ 2024}} & \multicolumn{1}{c}{\makecell{\textbf{iTransformer} \\ Liu \textit{et al.}, \\ 2023}} & \multicolumn{1}{c}{\makecell{\textbf{PatchTST} \\ Nie \textit{et al.}, \\ 2022}} & \multicolumn{1}{c}{\makecell{\textbf{TSMixer} \\ Ekambaram \\  \textit{et al.}, 2023 }} & \multicolumn{1}{c}{\makecell{\textbf{Crossformer} \\ Zhang and Yan, \\ 2023}} & \multicolumn{1}{c}{\makecell{\textbf{TiDE} \\ Das \textit{et al.}, \\ 2023}} & \multicolumn{1}{c}{\makecell{\textbf{DLinear} \\ Zeng \textit{et al.}, \\ 2023}} & \multicolumn{1}{c}{\makecell{\textbf{FEDformer} \\ Zhou \textit{et al.}, \\ 2022}} \\
        \cline{2-10}
        & MSE($\mathbb{\downarrow}$) / MAE($\mathbb{\downarrow}$) & MSE / MAE & MSE / MAE & MSE / MAE & MSE / MAE & MSE / MAE & MSE / MAE & MSE / MAE & MSE / MAE \\
        \midrule
        \multirow{1}{*}{ETTm1} & \textcolor{red}{\textbf{0.386}} / \textcolor{red}{\textbf{0.393}} & \textcolor{blue}{\underline{0.393}} / \textcolor{blue}{\underline{0.403}} & 0.407 / 0.410 & 0.396 / 0.406 & 0.398 / 0.407 & 0.513 / 0.496 & 0.419 / 0.419 & 0.474 / 0.453 & 0.543 / 0.490 \\
        \midrule
        \multirow{1}{*}{ETTm2} & \textcolor{red}{\textbf{0.276}} / \textcolor{red}{\textbf{0.319}} & \textcolor{blue}{\underline{0.287}} / \textcolor{blue}{\underline{0.330}} & 0.288 / 0.332 & \textcolor{blue}{\underline{0.287}} / \textcolor{blue}{\underline{0.330}} & 0.289 / 0.333 & 0.757 / 0.610 & 0.358 / 0.404 & 0.350 / 0.401 & 0.305 / 0.349 \\
        \midrule
        \multirow{1}{*}{ETTh1} & \textcolor{red}{\textbf{0.438}} / \textcolor{red}{\textbf{0.434}} & 0.449 / \textcolor{blue}{\underline{0.442}} & 0.454 / 0.447 & 0.453 / 0.446 & 0.463 / 0.452 & 0.529 / 0.522 & 0.541 / 0.507 & 0.456 / 0.452 & \textcolor{blue}{\underline{0.440}} / 0.460 \\
        \midrule
        \multirow{1}{*}{ETTh2} & \textcolor{blue}{\underline{0.374}} / \textcolor{red}{\textbf{0.397}} & \textcolor{red}{\textbf{0.373}} / \textcolor{blue}{\underline{0.400}} & 0.383 / 0.407 & 0.385 / 0.410 & 0.401 / 0.417 & 0.942 / 0.684 & 0.611 / 0.550 & 0.559 / 0.515 & 0.437 / 0.449 \\
        \midrule
        \multirow{1}{*}{ECL} & \textcolor{red}{\textbf{0.168}} / \textcolor{red}{\textbf{0.259}} & \textcolor{blue}{\underline{0.174}} / \textcolor{blue}{\underline{0.264}} & 0.178 / 0.270 & 0.189 / 0.276 & 0.186 / 0.287 & 0.244 / 0.334 & 0.251 / 0.344 & 0.212 / 0.300 & 0.214 / 0.327 \\
        \midrule
        \multirow{1}{*}{Traffic} & \textcolor{blue}{\underline{0.421}} / \textcolor{red}{\textbf{0.261}} & \textcolor{red}{\textbf{0.409}} / \textcolor{blue}{\underline{0.267}} & 0.428 / 0.282 & 0.454 / 0.286 & 0.522 / 0.357 & 0.550 / 0.304 & 0.760 / 0.473 & 0.625 / 0.383 & 0.610 / 0.376 \\
        \midrule
        \multirow{1}{*}{Weather} & \textcolor{red}{\textbf{0.243}} / \textcolor{red}{\textbf{0.267}} & \textcolor{blue}{\underline{0.255}} / \textcolor{blue}{\underline{0.278}} & 0.258 / \textcolor{blue}{\underline{0.278}} & 0.256 / 0.279 & 0.256 / 0.279 & 0.259 / 0.315 & 0.271 / 0.320 & 0.265 / 0.317 & 0.309 / 0.360 \\
        \midrule
        \multirow{1}{*}{Solar} & \textcolor{red}{\textbf{0.221}} / \textcolor{red}{\textbf{0.244}} & \textcolor{blue}{\underline{0.229}} / \textcolor{blue}{\underline{0.256}} & 0.233 / 0.262 & 0.236 / 0.266 & 0.260 / 0.297 & 0.641 / 0.639 & 0.347 / 0.417 & 0.330 / 0.401 & 0.291 / 0.381 \\
        \midrule
        \textbf{$1^{st} / 2^{nd}$} count & 14 / 2 & 2 / 13 & 0 / 1 & 0 / 2 & 0 / 0 & 0 / 0 & 0 / 0 & 0 / 0 & 0 / 1 \\
        \bottomrule
    \end{tabular}
    \caption{Multivariate forecasting results with horizon $H\in \{ 96, 192, 336, 720 \}$ and fixed lookback window length $L = 96$. Red values represent the best performance, while underlined values represent the second-best performance. Results are averaged from all prediction horizons. Full results are listed in Table 6. (Appendix B)}
    \label{tab:ett_comparison}
\end{table*}

\begin{figure*}[h!]
    \centering
    \includegraphics[width=0.95\textwidth]{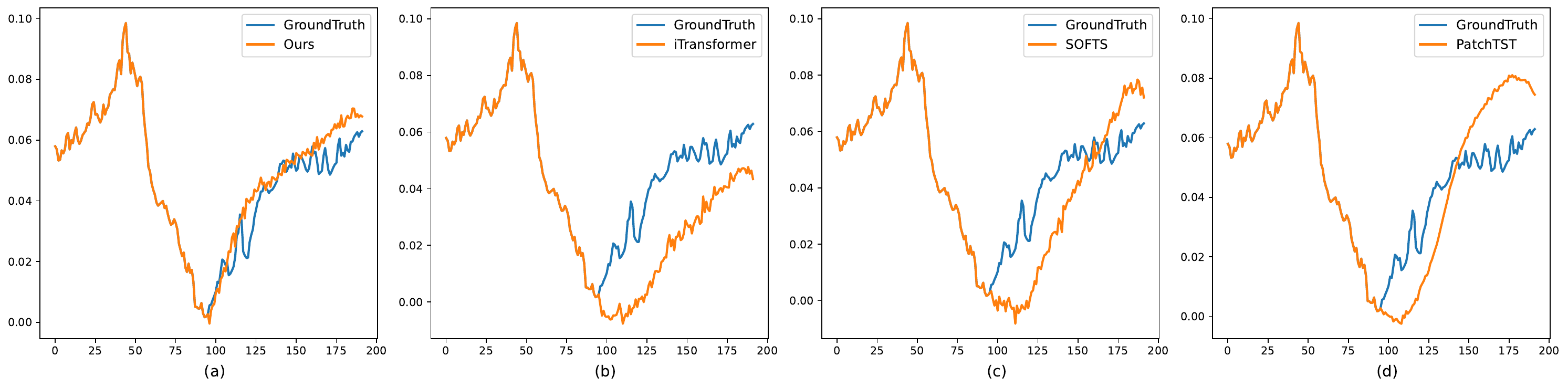} % 한 컬럼 너비로 설정
    \caption{Results of prediction of Ours and baseline models on Weather dataset. (a) CASA (b) iTransformer (c) SOFTS (d) PatchTST}
    \label{fig:prediction}
\end{figure*}

\paragraph{Complexity Analysis}
CASA is an efficient algorithm that exhibits \textbf{linear complexity} not only with respect to the number of tokens, i.e., the number of variates $N$, but also with respect to the input length $L$ and prediction length $H$. The detailed complexity calculation is as follows.
Let the kernel size of the score network be denoted as $k$. The complexities of the Reversible Instance Normalization (RevIN), series embedding, and MLP are $O(NL)$, $O(NLD)$, and $O(ND^2)$, respectively. Additionally, the complexity of the score network, composed of CNN autoencoder blocks, is $O(NkD^2)$. The predictor has a complexity of $O(NDH)$.
Thus, the overall complexity of our method is $O(NL + NLD + ND^2 + NkD^2 + NDH)$, which scales linearly with respect to $N$, $L$, and $H$. Since the hidden dimension and kernel size are constants in the algorithm, they can be ignored. Consequently, $N$ is dominated by $NL$ and $NH$(Since $L$ and $H$ typically take on large values in LTSF), leading to the final complexity summarized in Table \ref{tab:complexity_comparison}. In addition, the results for memory usage and inference time are included in the table, empirically demonstrating the efficiency of CASA. For details of the implementation, refer to the Appendix D.

\begin{table*}[ht]
    \centering
    \renewcommand{\arraystretch}{1.2} % 행 높이 조정
    \setlength{\tabcolsep}{4.0pt} % 열 사이 여백 늘리기
    \small  % 폰트 크기 조정
    \begin{tabular}{c c cc cc cc cc cc cc cc@{}} % 가장 오른쪽의 세로줄 제거 (@{} 사용)
        \toprule
        \multirow{2}{*}{Model} & \multirow{2}{*}{Comp} & 
        \multicolumn{2}{c}{ECL} & \multicolumn{2}{c}{Traffic} & \multicolumn{2}{c}{Weather} & \multicolumn{2}{c}{ETTm1} & \multicolumn{2}{c}{ETTm2} & \multicolumn{2}{c}{ETTh1} & \multicolumn{2}{c}{ETTh2} \\ % 마지막 세로줄 제거
        \cline{3-16}
        & & MSE($\mathbb{\downarrow}$) & MAE($\mathbb{\downarrow}$) & MSE & MAE & MSE & MAE & MSE & MAE & MSE & MAE & MSE & MAE & MSE & MAE \\
        \midrule
        \multirow{2}{*}{Transformer} & Attention & 0.203 & 0.292 & 0.655 & 0.359 & 0.245 & 0.296 & 0.407 & 0.417 & 0.369 & 0.398 & 0.482 & 0.465 & 0.522 & \textbf{0.481} \\
        & CASA & \textbf{0.201} & \textbf{0.287} & \textbf{0.645} & \textbf{0.354} & \textbf{0.242} &  \textbf{0.293} & \textbf{0.390} & \textbf{0.411} & \textbf{0.368} & \textbf{0.388} & \textbf{0.465} & \textbf{0.461} & \textbf{0.512} & 0.490 \\
        \midrule
        \multirow{2}{*}{PatchTST} & Attention & 0.189 & 0.276 & 0.454 & 0.286 & 0.256 & 0.279 & 0.396 & 0.406 & 0.287 & \textbf{0.330} & 0.453 & 0.446 & 0.385 & 0.410 \\
        & CASA & \textbf{0.186} & \textbf{0.273} & \textbf{0.440} & \textbf{0.280} & \textbf{0.253} & \textbf{0.277} & \textbf{0.386} & \textbf{0.402} & \textbf{0.285} & 0.332 & \textbf{0.452} & \textbf{0.443} & \textbf{0.365} & \textbf{0.399} \\
        \midrule
        \multirow{2}{*}{iTransformer} & Attention & 0.178 & 0.270 & 0.428 & 0.282 & 0.258 & 0.278 & 0.407 & 0.410 & 0.288 & 0.332 & 0.454 & 0.447 & 0.383 & 0.407 \\
        & CASA & \textbf{0.168} & \textbf{0.259} & \textbf{0.421} & \textbf{0.261} & \textbf{0.244} & \textbf{0.267} & \textbf{0.386} & \textbf{0.393} & \textbf{0.276} & \textbf{0.319} & \textbf{0.438} & \textbf{0.434} & \textbf{0.374} & \textbf{0.397} \\
        \bottomrule
    \end{tabular}
    \caption{The performance of CASA across three distinct Transformer-based models, each employing different tokenization techniques. The standard self-attention module is replaced with our CASA. Among the 42 metrics assessed, CASA demonstrated improvements in 40 of them.}
    \label{tab:model_generalization}
\end{table*}

\section{Experiments}

\paragraph{Dataset} We conduct our comprehensive experiment on 8 benchmark datasets \cite{zhou2021informer}, such as Traffic, ETT series including 4 subsets (ETTh1, ETTh2, ETTm1, ETTm2), Weather, Solar, Electricity. More detailed information on Dataset is described in Appendix A.

%The limited size of datasets in the LTSF domain reflects the inherent constraints of real-world benchmark datasets. However, the use of synthetic data is not feasible due to critical concerns about data reliability and authenticity \cite{hao2401synthetic,gonzales2023synthetic}. These limitations highlight the importance of studying LTSF under such constraints, as it addresses practical challenges and ensures that the models remain applicable.
%

\paragraph{Baselines} We chose totally 8 comtemporarlly baseline models, including SOFTS \cite{han2024softs}, iTransformer \cite{liu2023itransformer}, PatchTST \cite{nie2022time} , TSMixer \cite{ekambaram2023tsmixer}, Crossformer \cite{zhang2023crossformer}, TiDE \cite{das2023long} , DLinear \cite{zeng2023transformers},  FEDformer \cite{zhou2022fedformer}.

\begin{figure*}[h!]
    \centering
    \includegraphics[width=0.95\textwidth]{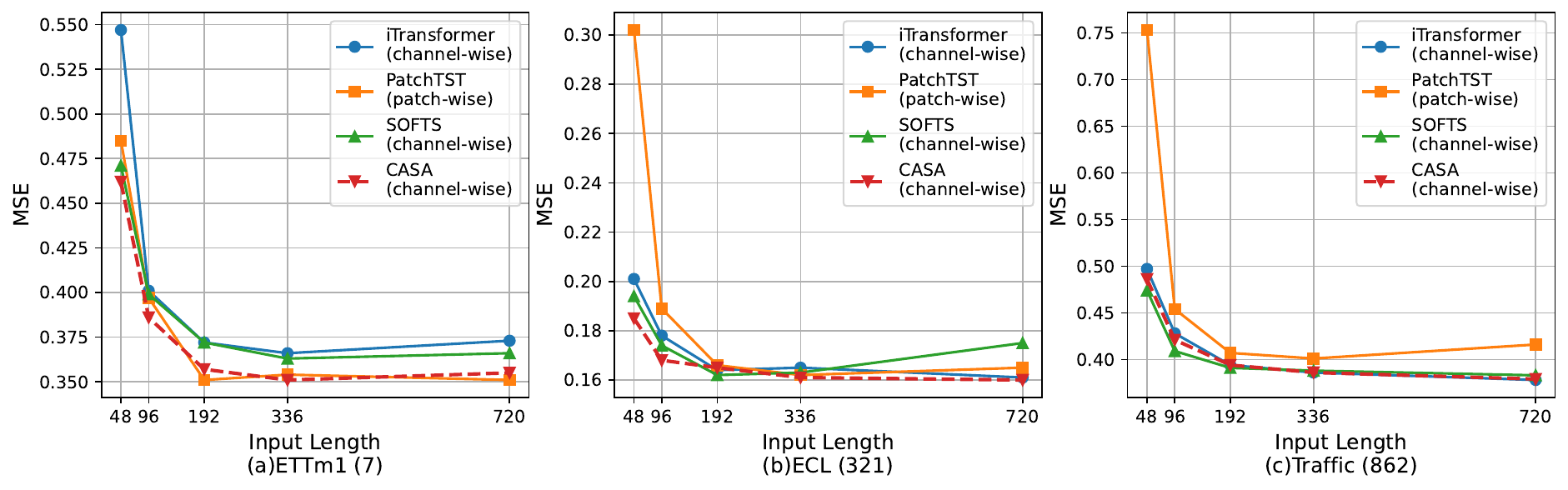} % 한 컬럼 너비로 설정
    \caption{Experimental results on the ETTm1, Electricity, and Traffic datasets (with 7, 321, and 862 variates, respectively). Our CASA remains robust across varying input and prediction lengths (48 to 720). Unlike PatchTST, which struggles as the number of variates increases, models like iTransformer and SOFTS, which tokenize variates, exhibit stronger performance.}
    \label{fig:input length ablation}
\end{figure*}

\paragraph{Setup}
Our comprehensive experiments results are based on MSE (Mean Squared Error) and MAE (Mean Absolute Error) metrics. Our main experiments are conducted on the conditions with $L=96$ and the $H \in \{96, 192, 336, 720\}$.

\subsection{Multivariate Forecasting Results}
The main results are presented in Table \ref{tab:ett_comparison}, where red-bold text indicates the best score and blue-underlined text represents the second-best score. CASA demonstrates the lowest MSE and MAE losses across 8 benchmark datasets, surpassing the previous state-of-the-art model, SOFTS, by a substantial margin. Additionally, the second-lowest performance scores exhibit a smaller gap from the best score compared to the others. Notably, the proposed model showcases its robustness on relatively large datasets, such as Traffic, Weather, and Solar, highlighting its ability to capture complex correlations, which significantly enhances the model's predictive performance. 
\begin{figure*}[h!]
    \centering
    \includegraphics[width=0.95\textwidth]{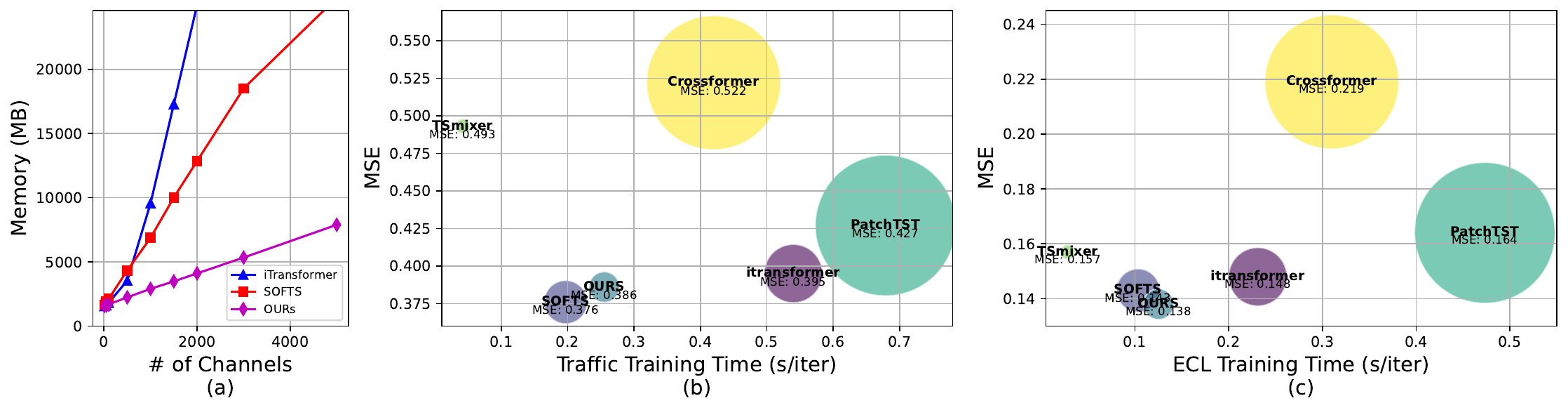} % 한 컬럼 너비로 설정
    \caption{(a) Memory usage scaling with the number of tokens, demonstrating CASA's linear growth and reduced memory consumption compared to SOFTS. (b, c) Experimental results on Traffic and Electricity datasets (batch size: 16, input/prediction length: 96), highlighting CASA's low memory usage and balanced trade-off between speed and performance.}
    \label{fig:efficiency}
\end{figure*}

Figure \ref{fig:prediction} visualizes the prediction performance of weather dataset from CASA, SOFTS, iTransformer, and PatchTST models against the ground truth. CASA shows the closest alignment with the label, while iTransformer shows similar tendency along the ground truth with large deviation. SOFTS and PatchTST prediction slightly detours the label. The full results of different prediction lengths and the visualization results on the rest of the datasets are demonstrated in Appendix B and Appendix C, respectively. 

\subsection{Superiority Analysis of CASA}
\paragraph{Replacing Self-Attention with CASA}
To ensure the proposed model's adaptability to diverse tokenization techniques, we integrate CASA into various Transformer variants including the vanilla Transformer, PatchTST, and iTransformer. Especially for the vanilla Transformer, we only use the encoder architecture to appropriately compare the effectiveness of CASA. 

The results on seven benchmark datasets are presented in Table \ref{tab:model_generalization}. With the exception of two specific cases—ETTm2 compared to PatchTST and ETTh2 compared to the vanilla Transformer—CASA consistently enhances the performance of the original and variants models, achieving improvements in 40 out of 42 results across the benchmarks. These findings not only demonstrate that replacing self-attention with CASA significantly boosts forecasting accuracy, but also highlight its flexibility and adaptability, as it can seamlessly integrate with diverse tokenization techniques, making it a versatile enhancement for Transformer-based architectures.
\vspace{-3pt}
\begin{figure}[h!]
    \centering
    \includegraphics[width=0.75\columnwidth]{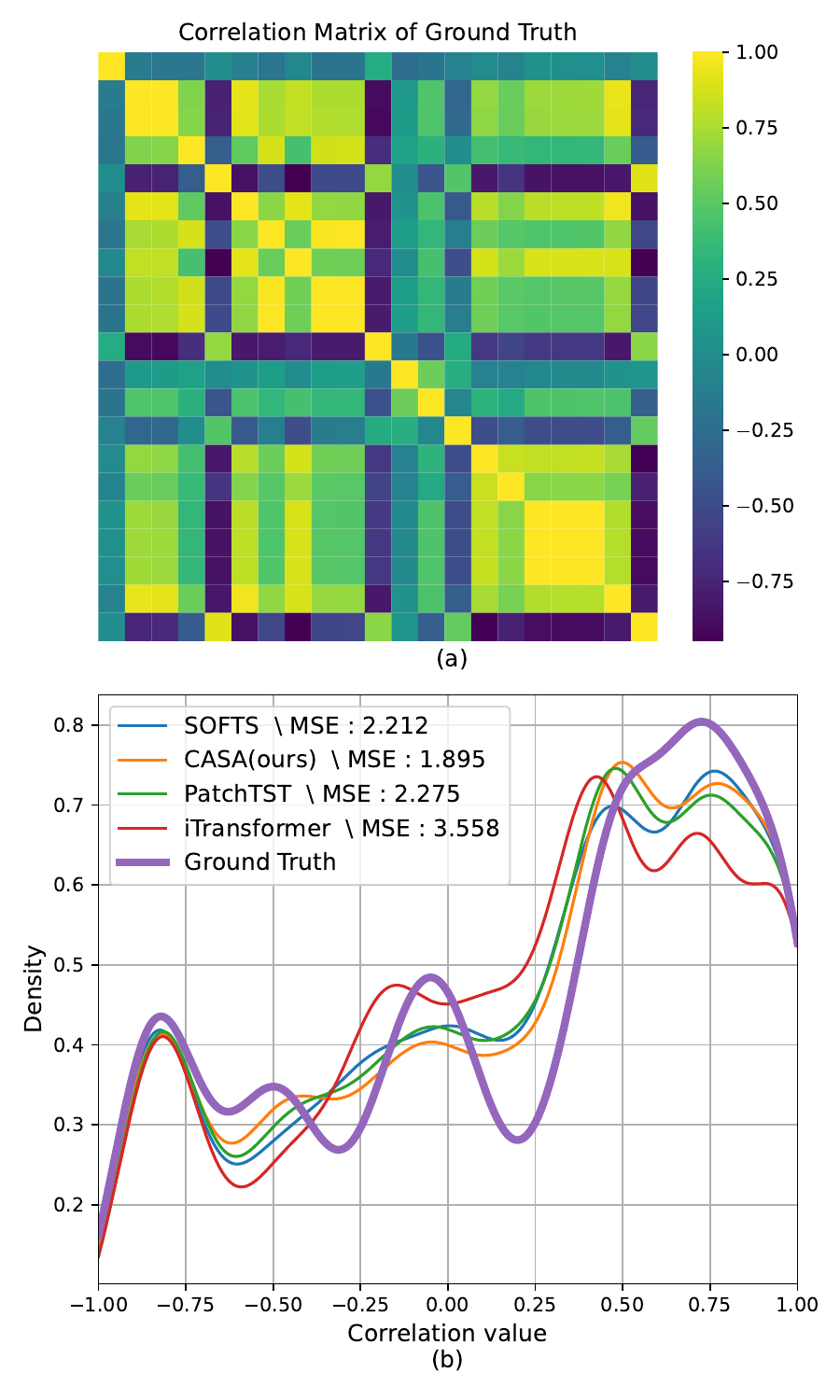} % 한 컬럼 너비로 설정
    \caption{(a) Correlation matrix among the Ground Truth variates, computed on the Weather Dataset. (b) PDFs of the Ground Truth and each model, computed using KDE.}
    \label{fig:corr_merge}
\end{figure}

\begin{table}[h!]
    \centering
    \renewcommand\thesubfigure{(\alph{subfigure})}
    \setlength{\tabcolsep}{3.0pt}
    \setlength{\abovecaptionskip}{0.2pt} % 캡션 상단 간격을 5pt로 설정
    \setlength{\belowcaptionskip}{0.2pt} % 캡션 하단 간격을 5pt로 설정
    \begin{tabular}{lcccc}
        \toprule
        \textbf{Model} & \textbf{MSE($\mathbb{\downarrow}$)} & \textbf{Cosine Similarity($\mathbb{\uparrow}$)} & \textbf{SSIM($\mathbb{\uparrow}$)} \\
        \midrule
        CASA(ours) & \textbf{1.1204} & \textbf{0.9965} & \textbf{0.9912} \\
        SOFTS & \underline{1.2520} & 0.9957 & 0.9889 \\
        iTransformer & 1.8675 & 0.9910 & 0.9791 \\
        PatchTST & 1.2393 & \underline{0.9960} & \underline{0.9896} \\
    \bottomrule
    \end{tabular}
    \caption{Metrics for the correlation matrices of the Ground Truth and each model. Boldface indicates the best performance, and underlining indicates the second-best performance.}
    \label{tab:corr_matrix2}
\end{table}

\paragraph{Robustness of CASA under varying conditions}
 To validate the robustness of CASA with respect to input length and the number of variates, we conducted experiments on the ETTm1, Electricity, and Traffic datasets, which contain 7, 321, and 862 variates, respectively, with prediction lengths ranging from 48 to 720. As shown in Figure \ref{fig:input length ablation}, the performance of models utilizing non-channel-wise tokenization declines as the number of variates increases. Specifically, PatchTST experiences a significant performance drop on the Traffic dataset, recording the highest MSE losses. In contrast, models such as iTransformer and SOFTS, which employ channel-wise tokenization, demonstrate greater resilience to increases in the number of variates. However, both models exhibit elevated MSE losses on the ETTm1 dataset, while their performance improves on the Electricity and Traffic datasets. In comparison, our proposed model maintains stable MSE losses across all three datasets, achieving notably lower MSE losses on the ETTm1 dataset. This underscores CASA's ability to deliver high predictive accuracy under varying conditions, ensuring consistent performance even as the input data length increases.

\subsection{Model Efficiency Analysis}\label{sec:EfficiencyAnalysis}

In this section, we empirically validate the efficiency of CASA, as theoretically outlined in Section \ref{sec:theoritical proof}. For comparison, we use iTransformer and SOFTS as baselines. Figure \ref{fig:efficiency} (a) depicts memory usage, revealing that CASA exhibits linear complexity and effectively leverages practical computational resources. This performance is comparable to SOFTS, which also demonstrates similar complexity but increases memory usage significantly. Notably, CASA significantly outperforms iTransformer, which suffers from quadratic memory growth. Figures \ref{fig:efficiency} (b) and (c) present memory footprints, inference time, and MSE for the Traffic and Electricity datasets, using a batch size of 16 and input/inference sequence lengths of 96. CASA consumes fewer computational resources than Transformer variants such as iTransformer, Crossformer, and PatchTST. Regarding MSE, CASA achieves the second-lowest value on Traffic and the lowest on Electricity, all while maintaining fast inference and efficient resource usage. Although CASA slightly exceeds TSMixer in memory usage, it delivers stronger overall performance, striking an effective balance between accuracy and resource efficiency.

\vspace{-3pt}

\subsection{Investigating Cross-Dimensional Interactions: A Correlation Matrix Analysis of CASA}\label{sec:dist}

We evaluate CASA’s capacity to capture cross-dimensional interactions by examining the correlation matrices derived from each model’s predictions (i.e., correlations among all variates) on the Weather Dataset, comparing outcomes from CASA, SOFTS, iTransformer, and PatchTST against the Ground Truth. The Ground Truth correlation matrix shows distinct positive and negative correlation blocks, indicating a clear spatial structure with a grid-like arrangement (Figure \ref{fig:corr_merge} (a)). Furthermore, kernel density estimation (KDE) using a Gaussian kernel reveals that correlation values are mostly concentrated in the positive domain (Figure \ref{fig:corr_merge} (b)), suggesting many variates rise or fall in sync. Notably, CASA’s correlation matrix most closely approximates the Ground Truth, as evidenced by the lowest MSE loss between their probability density functions (PDFs) (Figure \ref{fig:corr_merge} (b)). This underscores CASA’s ability to preserve intricate relationships essential for capturing temporal and spatial dependencies in weather data. For a more comprehensive evaluation, we compare each model’s correlation matrix to the Ground Truth using Mean Squared Error (MSE), cosine similarity, and the Structural Similarity Index Measure (SSIM). CASA outperforms all other models across these metrics, validating its effectiveness in capturing cross-dimensional interactions (Table \ref{tab:corr_matrix2}).

\vspace{-3pt}

\section{Conclusion}
In this study, we introduce the CASA model, which demonstrates remarkable effectiveness and solid predictive performance, as confirmed by a broad range of experiments. By delivering state-of-the-art results while requiring notably fewer computational resources and less processing time than existing approaches. Moreover, its CNN-based autoencoder module successfully captures cross-dimensional interactions throughout the compress-and-decompress procedure, which contributes to its outstanding performance. Additionally, CASA shows strong potential as a versatile alternative to conventional attention mechanisms in various Transformer configurations, remaining unaffected by different tokenization methods. This further highlights its adaptability, practicality, and efficiency across diverse use cases.

%\section*{Ethical Statement}

%There are no ethical issues.

\vspace{-3pt}

\section*{Acknowledgments}
%This work is supported by the National Research Foundation of Korea (NRF) grant and funded by the Korea government(MSIT)/Institute of %Information \& communications Technology Planning \& Evaluation (IITP)
% (RS-2024-00421203, RS-2024-00406127, 2021R1A2C3010887, RS-2021-II211343 (Artificial Intelligence 
%Graduate School Program (Seoul National University)))

This work was partly supported by Institute of Information \& communications Technology Planning \& Evaluation (IITP) grant funded by the Korea government (MSIT) [NO.RS-2021-II211343, Artificial Intelligence Graduate School Program (Seoul National University)] and the National Research Foundation of Korea (NRF) [RS-2024-00421203, RS-2024-00406127, RS-2021-NR059802]

\bibliographystyle{named}
\bibliography{ijcai24}

%\end{comment}

\clearpage
%\begin{comment}
\vspace{1em}
\noindent\textbf{\LARGE Appendix}
\vspace{1em}
\renewcommand{\thesection}{\Alph{section}}
\setcounter{section}{0}
\section{Dataset} \label{Apen:A}
\subsection{Details of Each Dataset}
Data preprocessing is largely based on previous work SOFTS. We provide the dataset descriptions:

\begin{itemize}
    \item \textbf{ETT (Electricity Transformer Temperature)}: This dataset includes two hourly (ETTh) and two 15-minute (ETTm) variations, containing seven features related to oil and load measurements of electricity transformers, collected from July 2016 to July 2018.
    \item \textbf{Traffic}: This dataset represents hourly road occupancy rates gathered by sensors monitoring San Francisco freeways from 2015 to 2016.
    \item \textbf{Electricity}: It consists of hourly electricity usage data for 321 consumers, spanning from 2012 to 2014.
    \item \textbf{Weather}: it contains 21 weather-related variables, such as temperature and humidity, recorded every 10 minutes throughout 2020 in Germany.
    \item \textbf{Solar-Energy}: It captures solar power generation data from 137 photovoltaic plants in 2006, sampled every 10 minutes.
\end{itemize}
Further details of these datasets are summarized in Table \ref{tab:datasets}.

\begin{table}[h!]
    \centering
    \begin{tabular}{lcccc}
        \hline
        \textbf{Dataset} & \textbf{Timesteps} & \textbf{Sample Frequency} & \textbf{Dimension} \\
        \hline
        Weather & 52,696 & 10 min & 21 \\
        Electricity & 26,304 & 1 hour & 321 \\
        Traffic & 17,544 & 1 hour & 862 \\
        Solar & 52,560 & 10 min & 137 \\
        ETTh1 & 17,420 & 1 hour & 7 \\
        ETTh2 & 17,420 & 1 hour & 7 \\
        ETTm1 & 69,680 & 15 min & 7 \\
        ETTm2 & 69,680 & 15 min & 7 \\
        \hline
    \end{tabular}
    \caption{Characteristics of the datasets including timesteps, sample frequency, and dimensions. Dimensions denotes the number of variate in each dataset.}
    \label{tab:datasets}
\end{table}

\subsection{Properties of datasets}
The limited size of datasets in the LTSF domain reflects the inherent constraints of real-world benchmark datasets. However, the use of synthetic data is not feasible due to critical concerns about data reliability and authenticity \cite{hao2401synthetic,gonzales2023synthetic}. These limitations highlight the importance of studying LTSF under such constraints, as it addresses practical challenges and ensures that the models remain applicable.

\section{Full Results} \label{Apen:B}

We fix the input length to $L=96$ and evaluated the prediction results across 8 real-world datasets for $H \in \{96, 192, 336, 720 \}$. For each dataset, four MSE and MAE is calculated, resulting in a total of 64 metrics. Among these, CASA achieves the best performance in 54 cases (84.4\%). Detailed results are presented in Table \ref{tab:full results}.

\section{Visualization Results} \label{Apen:C}
Figure \ref{fig:prediction_etth1} to Figure \ref{fig:prediction_ecl} visualize the prediction performance of weather dataset from CASA, SOFTS, iTransformer, and PatchTST models against the ground truth. The both input length and the prediction length are fixed to 96.

\begin{figure*}[h!]
    \centering
    \includegraphics[width=0.95\textwidth]{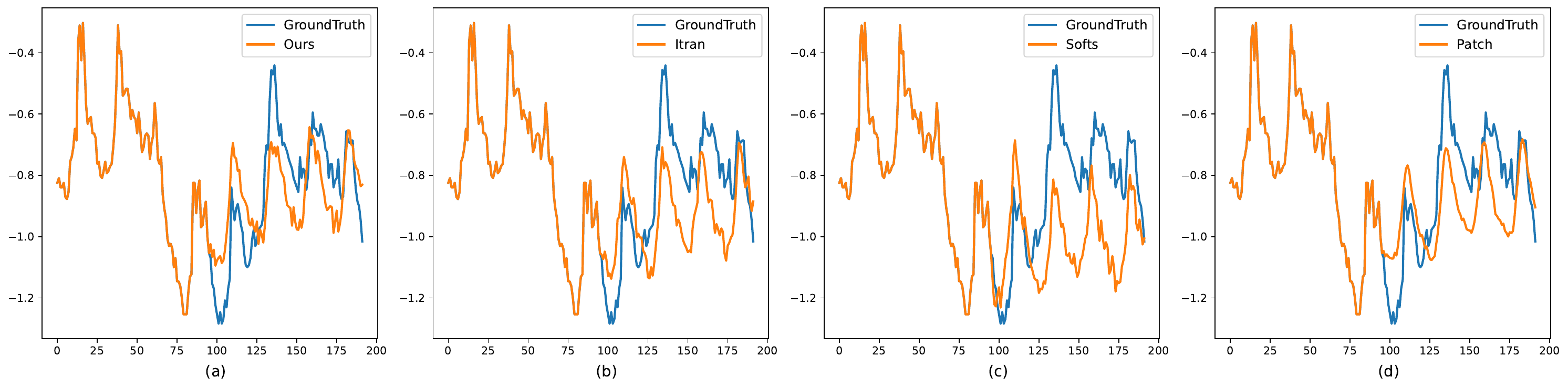} % 한 컬럼 너비로 설정
    \caption{Results of prediction of Ours and baseline models on Etth1 dataset}
    \label{fig:prediction_etth1}
\end{figure*}

\begin{figure*}[h!]
    \centering
    \includegraphics[width=0.95\textwidth]{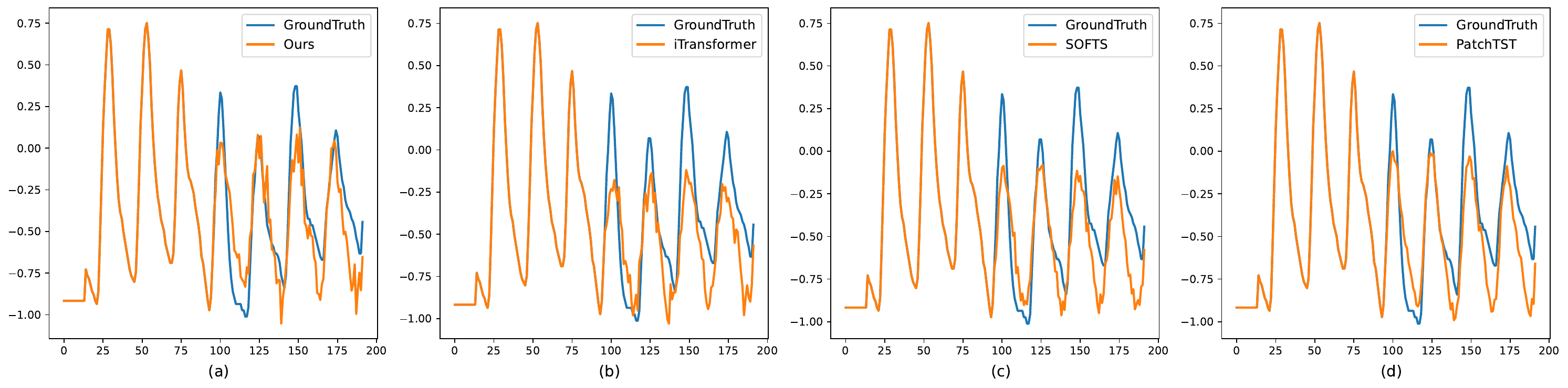} % 한 컬럼 너비로 설정
    \caption{Results of prediction of Ours and baseline models on Etth2 dataset}
    \label{fig:prediction_etth2}
\end{figure*}

\begin{figure*}[h!]
    \centering
    \includegraphics[width=0.95\textwidth]{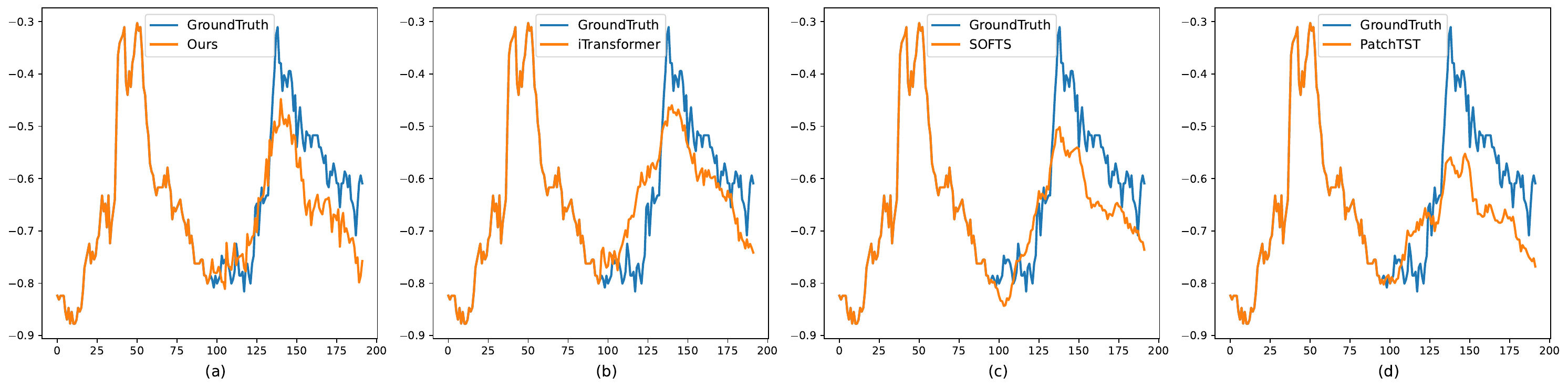} % 한 컬럼 너비로 설정
    \caption{Results of prediction of Ours and baseline models on Ettm1 dataset}
    \label{fig:prediction_ettm1}
\end{figure*}

\begin{figure*}[h!]
    \centering
    \includegraphics[width=0.95\textwidth]{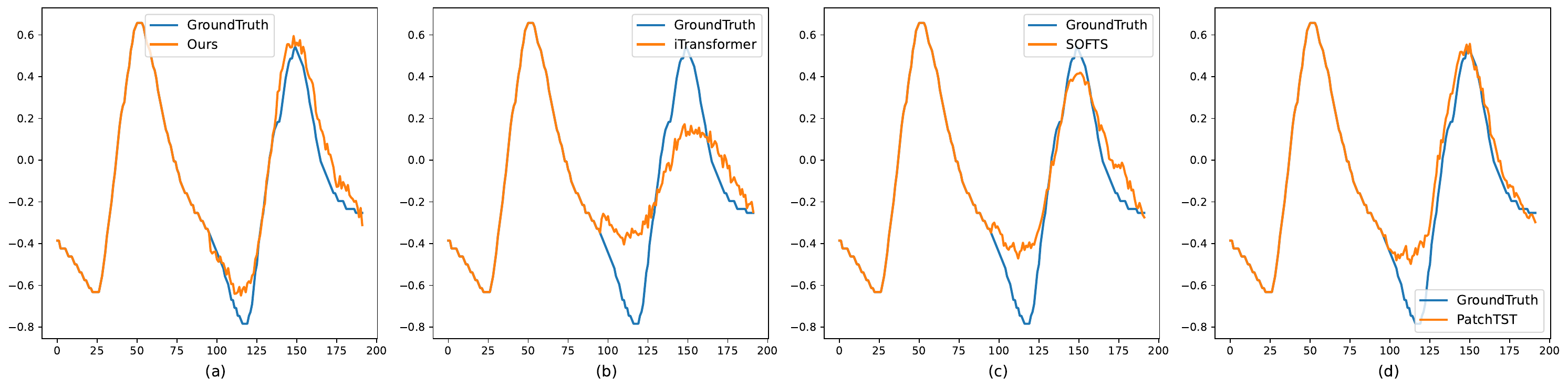} % 한 컬럼 너비로 설정
    \caption{Results of prediction of Ours and baseline models on Ettm2 dataset}
    \label{fig:prediction_ettm2}
\end{figure*}

\begin{figure*}[h!]
    \centering
    \includegraphics[width=0.95\textwidth]{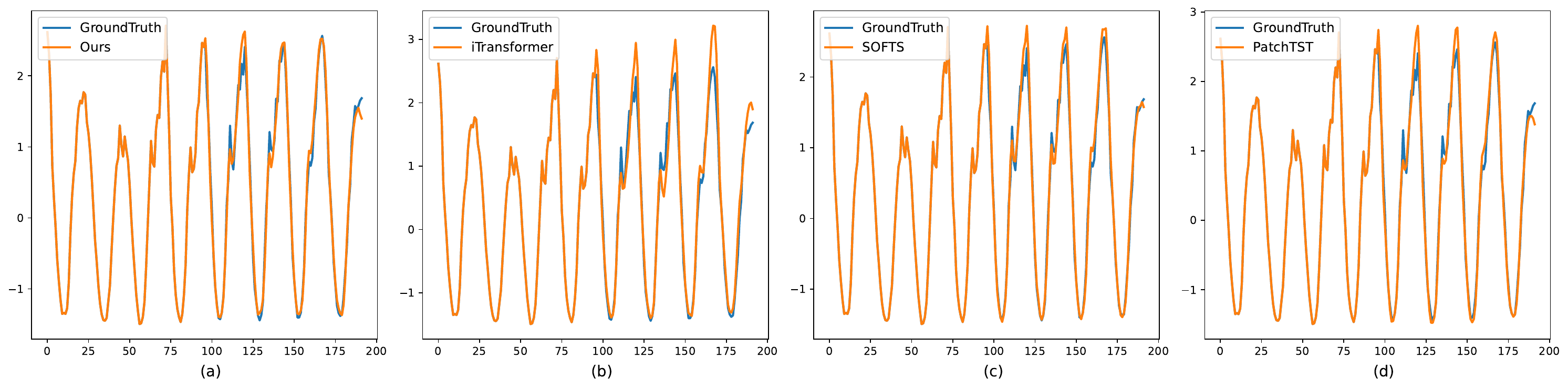} % 한 컬럼 너비로 설정
    \caption{Results of prediction of Ours and baseline models on Traffic dataset}
    \label{fig:prediction_traffic}
\end{figure*}

\begin{figure*}[h!]
    \centering
    \includegraphics[width=0.95\textwidth]{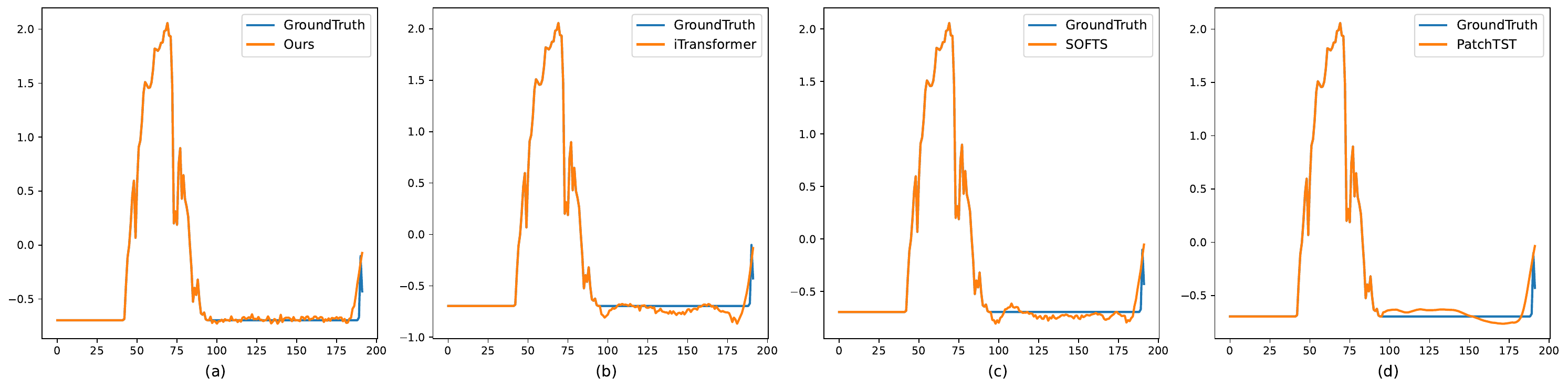} % 한 컬럼 너비로 설정
    \caption{Results of prediction of Ours and baseline models on Solar dataset}
    \label{fig:prediction_solar}
\end{figure*}

\begin{figure*}[h!]
    \centering
    \includegraphics[width=0.95\textwidth]{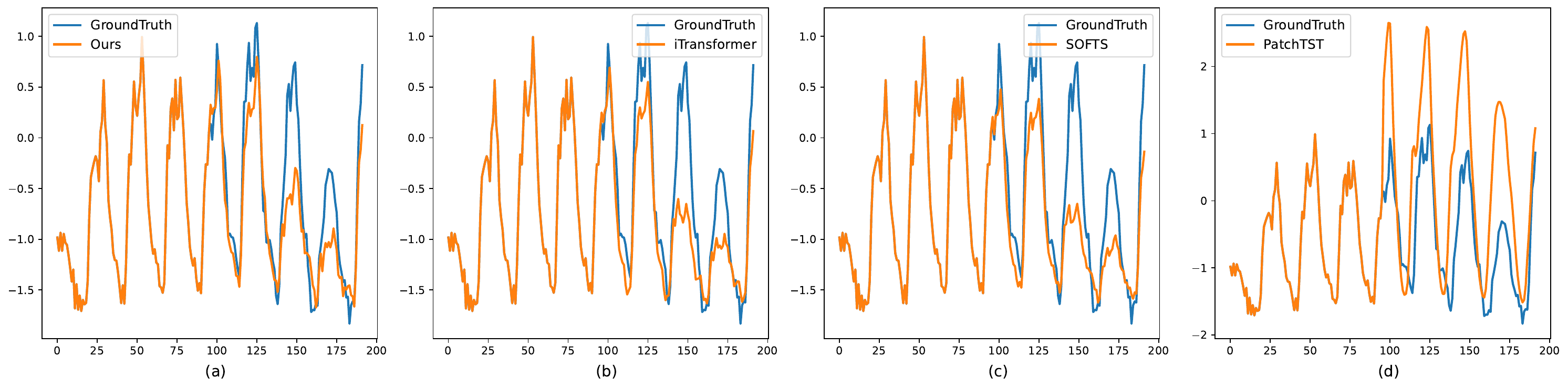} % 한 컬럼 너비로 설정
    \caption{Results of prediction of Ours and baseline models on Electricity dataset}
    \label{fig:prediction_ecl}
\end{figure*}

\section{Implementation Details in Section 3.4.} \label{Apen:D}
In the experiments conducted on the Traffic dataset, $N$ is set to 862. The parameters $L$ and $H$ are configured as 512 and 96, respectively, with a batch size of 16. CASA and SOFTS exhibit the same complexity with respect to $N$, $L$, and $H$, but their experimental results show slight variations due to differences in hyperparameter settings, such as the hidden dimension and kernel size. Except for CASA and SOFTS, all models include a quadratic term. However, iTransformer employs channel-wise tokenization, resulting in a quadratic term with respect to $N$, whereas PatchTST and Transformer exhibit a quadratic term with respect to $L$. Given that $1.5 \times L \leq N$ in the training setup, iTransformer is less efficient compared to Transformer. Additionally, PatchTST involves an $NL^2$ term, leading to the highest computational complexity as the input length and the number of variates increase.

\begin{table*}[ht]
    \centering
    \tiny  % 폰트 크기를 더 작게 조정
    \setlength{\tabcolsep}{2.5pt}  % 열 사이의 여백 줄이기
    \begin{tabular}{c|c|cc|cc|cc|cc|cc|cc|cc|cc|cc|cc|cc|cc}
        \toprule
        \multirow{2}{*}{Metric} & \multirow{2}{*}{Models} & 
        \multicolumn{2}{c|}{CASA(ours)} & \multicolumn{2}{c|}{SOFTS} & \multicolumn{2}{c|}{iTransformer} & \multicolumn{2}{c|}{PatchTST} & \multicolumn{2}{c|}{TSMixer} & \multicolumn{2}{c|}{Crossformer} & \multicolumn{2}{c|}{TiDE} & \multicolumn{2}{c|}{TimesNet} & \multicolumn{2}{c|}{Stationary} & \multicolumn{2}{c|}{DLinear} & \multicolumn{2}{c|}{SCINet} & \multicolumn{2}{c}{FEDformer} \\
        \cline{3-26}
        & & MSE & MAE & MSE & MAE & MSE & MAE & MSE & MAE & MSE & MAE & MSE & MAE & MSE & MAE & MSE & MAE & MSE & MAE & MSE & MAE & MSE & MAE & MSE & MAE \\
        \midrule
        \multirow{5}{*}{ETTm1} & 96 & \textbf{0.313} & \textbf{0.358} & 0.325 & \underline{0.361} & 0.334 & 0.368 & 0.329 & 0.365 & \underline{0.323} & 0.363 & 0.404 & 0.426 & 0.364 & 0.387 & 0.338 & 0.375 & 0.386 & 0.398 & 0.345 & 0.372 & 0.418 & 0.438 & 0.379 & 0.419 \\
        & 192 & \textbf{0.375} & \textbf{0.382} & 0.377 & 0.391 & 0.380 & 0.394 & 0.376 & 0.392 & 0.450 & 0.451 & 0.398 & 0.404 & \underline{0.374} & \underline{0.387} & 0.426 & 0.441 & 0.459 & 0.444 & 0.380 & 0.389 & 0.439 & 0.450 & 0.426 & 0.441 \\
        & 336 & \textbf{0.397} & \textbf{0.398} & 0.426 & 0.420 & \underline{0.400} & \underline{0.410} & 0.407 & 0.413 & 0.532 & 0.515 & 0.428 & 0.425 & 0.410 & 0.411 & 0.445 & 0.459 & 0.495 & 0.46 & 0.413 & 0.413 & 0.490 & 0.485 & 0.445 & 0.459 \\
        & 720 & \textbf{0.459} & \textbf{0.437} & \underline{0.466} & \underline{0.447} & 0.491 & 0.459 & 0.475 & 0.453 & 0.485 & 0.459 & 0.666 & 0.589 & 0.487 & 0.461 & 0.478 & 0.450 & 0.585 & 0.516 & 0.474 & 0.453 & 0.595 & 0.550 & 0.543 & 0.490 \\
        \cmidrule{2-26}
        & Avg & \textbf{0.386} & \textbf{0.393} & \underline{0.393} & \underline{0.403} & 0.407 & 0.410 & 0.396 & 0.406 & 0.398 & 0.407 & 0.513 & 0.496 & 0.419 & 0.419 & 0.400 & 0.406 & 0.481 & 0.456 & 0.474 & 0.453 & 0.595 & 0.550 & 0.543 & 0.490 \\
        \midrule
        \multirow{5}{*}{ETTm2} & 96 & \textbf{0.173} & \textbf{0.252} & \underline{0.180} & \underline{0.261} & \underline{0.180} & 0.264 & 0.184 & 0.264 & 0.182 & 0.266 & 0.287 & 0.366 & 0.207 & 0.305 & 0.187 & 0.267 & 0.192 & 0.274 & 0.193 & 0.292 & 0.286 & 0.377 & 0.203 & 0.287 \\
        & 192 & \textbf{0.240} & \textbf{0.298} & \underline{0.246} & \underline{0.306} & 0.250 & 0.309 & \underline{0.246} & \underline{0.306} & 0.249 & 0.309 & 0.414 & 0.492 & 0.290 & 0.364 & 0.249 & 0.309 & 0.280 & 0.339 & 0.284 & 0.362 & 0.399 & 0.445 & 0.269 & 0.328 \\
        & 336 & \textbf{0.296} & \textbf{0.334} & 0.319 & 0.352 & 0.311 & 0.348 & \underline{0.308} & \underline{0.346} & 0.309 & 0.347 & 0.597 & 0.542 & 0.377 & 0.422 & 0.321 & 0.351 & 0.334 & 0.361 & 0.369 & 0.427 & 0.637 & 0.591 & 0.325 & 0.366 \\
        & 720 & \textbf{0.395} & \textbf{0.392} & \underline{0.405} & \underline{0.401} & 0.412 & 0.407 & 0.409 & 0.402 & 0.416 & 0.408 & 1.730 & 1.042 & 0.558 & 0.524 & 0.408 & 0.403 & 0.417 & 0.413 & 0.554 & 0.522 & 0.960 & 0.735 & 0.421 & 0.415 \\
        \cmidrule{2-26}
        & Avg & \textbf{0.276} & \textbf{0.319} & \underline{0.287} & \underline{0.330} & 0.288 & 0.332 & \underline{0.287} & \underline{0.330} & 0.289 & 0.333 & 0.757 & 0.610 & 0.358 & 0.404 & 0.291 & 0.333 & 0.306 & 0.347 & 0.350 & 0.401 & 0.571 & 0.537 & 0.305 & 0.349 \\
        \midrule
        \multirow{5}{*}{ETTh1} & 96 & \textbf{0.376} & \textbf{0.397} & \underline{0.381} & \underline{0.399} & 0.386 & 0.405 & 0.394 & 0.406 & 0.401 & 0.412 & 0.423 & 0.448 & 0.479 & 0.464 & 0.384 & 0.402 & 0.513 & 0.491 & 0.386 & 0.400 & 0.654 & 0.599 & \textbf{0.376} & 0.419 \\
        & 192 & \underline{0.427} & \textbf{0.424} & 0.435 & \underline{0.431} & 0.441 & 0.436 & 0.440 & 0.435 & 0.452 & 0.442 & 0.471 & 0.474 & 0.525 & 0.492 & 0.436 & {0.429} & 0.534 & 0.504 & 0.437 & 0.432 & 0.719 & 0.631 & \textbf{0.420} & 0.448 \\
        & 336 & \underline{0.469} & \textbf{0.445} & 0.480 & \underline{0.452} & 0.487 & 0.458 & 0.491 & 0.462 & 0.492 & 0.463 & 0.570 & 0.546 & 0.565 & 0.515 & 0.491 & 0.469 & 0.588 & 0.535 & 0.481 & 0.459 & 0.778 & 0.659 & \textbf{0.459} & 0.465 \\
        & 720 & \textbf{0.479} & \textbf{0.468} & 0.499 & 0.488 & 0.503 & 0.491 & \underline{0.487} & \underline{0.479} & 0.507 & 0.490 & 0.653 & 0.621 & 0.594 & 0.558 & 0.521 & 0.500 & 0.643 & 0.616 & 0.519 & 0.516 & 0.836 & 0.699 & 0.506 & 0.507 \\
        \cmidrule{2-26}
        & Avg & \textbf{0.438} & \textbf{0.434} & 0.449 & \underline{0.442} & 0.454 & 0.447 & 0.453 & 0.446 & 0.463 & 0.452 & 0.529 & 0.522 & 0.541 & 0.507 & 0.458 & 0.450 & 0.570 & 0.537 & 0.456 & 0.452 & 0.747 & 0.647 & \underline{0.440} & 0.460 \\
        \midrule
        \multirow{5}{*}{ETTh2} & 96 & \underline{0.290} & \textbf{0.339} & 0.297 & 0.347 & 0.297 & 0.349 & \textbf{0.288} & \underline{0.340} & 0.319 & 0.361 & 0.745 & 0.584 & 0.400 & 0.440 & 0.340 & 0.374 & 0.476 & 0.458 & 0.333 & 0.387 & 0.707 & 0.621 & 0.358 & 0.397 \\
        & 192 & \textbf{0.366} & \textbf{0.388} & \underline{0.373} & \underline{0.394} & 0.380 & 0.400 & 0.376 & 0.395 & 0.402 & 0.410 & 0.877 & 0.656 & 0.528 & 0.509 & 0.402 & 0.414 & 0.512 & 0.493 & 0.477 & 0.476 & 0.860 & 0.689 & 0.429 & 0.439 \\
        & 336 & \underline{0.416} & \textbf{0.425} & \textbf{0.410} & \underline{0.426} & 0.428 & 0.432 & 0.440 & 0.451 & 0.444 & 0.446 & 1.043 & 0.731 & 0.643 & 0.571 & 0.452 & 0.452 & 0.552 & 0.551 & 0.594 & 0.541 & 1.000 & 0.744 & 0.496 & 0.487 \\
        & 720 & \underline{0.420} & \underline{0.439} & \textbf{0.411} & \textbf{0.433} & 0.427 & 0.445 & 0.436 & 0.453 & 0.441 & 0.450 & 1.104 & 0.763 & 0.874 & 0.679 & 0.462 & 0.468 & 0.562 & 0.560 & 0.831 & 0.657 & 1.249 & 0.838 & 0.463 & 0.474 \\
        \cmidrule{2-26}
        & Avg & \underline{0.374} & \textbf{0.397} & \textbf{0.373} & \underline{0.400} & 0.383 & \underline{0.407} & 0.385 & 0.410 & 0.401 & 0.417 & 0.942 & 0.684 & 0.611 & 0.550 & 0.414 & 0.427 & 0.526 & 0.516 & 0.559 & 0.515 & 0.954 & 0.723 & 0.437 & 0.449 \\
        \midrule
        \multirow{5}{*}{ECL} & 96 & \textbf{0.138} & \textbf{0.231} & \underline{0.143} & \underline{0.233} & 0.148 & 0.240 & 0.164 & 0.251 & 0.157 & 0.260 & 0.219 & 0.314 & 0.237 & 0.329 & 0.168 & 0.272 & 0.169 & 0.273 & 0.197 & 0.282 & 0.247 & 0.345 & 0.193 & 0.308 \\
        & 192 & \textbf{0.157} & \textbf{0.248} & \underline{0.158} & \textbf{0.248} & 0.162 & \underline{0.253} & 0.173 & 0.262 & 0.173 & 0.274 & 0.231 & 0.322 & 0.236 & 0.330 & 0.184 & 0.289 & 0.182 & 0.286 & 0.196 & 0.285 & 0.257 & 0.355 & 0.201 & 0.315 \\
        & 336 & \textbf{0.175} & \textbf{0.267} & \underline{0.178} & \underline{0.269} & \underline{0.178} & \underline{0.269} & 0.190 & 0.279 & 0.192 & 0.295 & 0.246 & 0.337 & 0.249 & 0.344 & 0.198 & 0.300 & 0.200 & 0.304 & 0.209 & 0.301 & 0.269 & 0.369 & 0.214 & 0.329 \\
        & 720 & \textbf{0.200} & \textbf{0.290} & \underline{0.218} & \underline{0.305} & 0.225 & 0.317 & 0.230 & 0.313 & 0.223 & 0.318 & 0.280 & 0.363 & 0.284 & 0.373 & 0.220 & 0.320 & 0.222 & 0.321 & 0.245 & 0.333 & 0.299 & 0.390 & 0.246 & 0.355 \\
        \cmidrule{2-26}
        & Avg & \textbf{0.168} & \textbf{0.259} & \underline{0.174} & \underline{0.264} & 0.178 & 0.270 & 0.189 & 0.276 & 0.186 & 0.287 & 0.244 & 0.334 & 0.251 & 0.344 & 0.192 & 0.295 & 0.193 & 0.296 & 0.212 & 0.300 & 0.268 & 0.365 & 0.214 & 0.327 \\
        \midrule
        \multirow{5}{*}{Traffic} & 96 & \underline{0.386} & \textbf{0.243} & \textbf{0.376} & \underline{0.251} & 0.395 & 0.268 & 0.427 & 0.272 & 0.493 & 0.336 & 0.522 & 0.290 & 0.805 & 0.493 & 0.593 & 0.321 & 0.612 & 0.338 & 0.650 & 0.396 & 0.788 & 0.499 & 0.587 & 0.366 \\
        & 192 & \underline{0.410} & \textbf{0.255} & \textbf{0.398} & \underline{0.261} & 0.417 & 0.276 & 0.454 & 0.289 & 0.497 & 0.351 & 0.530 & 0.293 & 0.756 & 0.474 & 0.617 & 0.336 & 0.613 & 0.340 & 0.598 & 0.370 & 0.789 & 0.505 & 0.604 & 0.373 \\
        & 336 & \underline{0.426} & \textbf{0.263} & \textbf{0.415} & \underline{0.269} & 0.433 & 0.283 & 0.450 & 0.282 & 0.528 & 0.361 & 0.558 & 0.305 & 0.762 & 0.477 & 0.629 & 0.336 & 0.618 & 0.328 & 0.605 & 0.373 & 0.797 & 0.508 & 0.621 & 0.383 \\
        & 720 & \underline{0.461} & \textbf{0.282} & \textbf{0.447} & \underline{0.287} & 0.467 & 0.302 & 0.484 & 0.301 & 0.569 & 0.380 & 0.589 & 0.328 & 0.719 & 0.449 & 0.640 & 0.350 & 0.653 & 0.355 & 0.645 & 0.394 & 0.841 & 0.523 & 0.626 & 0.382 \\
        \cmidrule{2-26}
        & Avg & \underline{0.421} & \textbf{0.261} & \textbf{0.409} & \underline{0.267} & 0.428 & 0.282 & 0.454 & 0.286 & 0.522 & 0.357 & 0.550 & 0.304 & 0.760 & 0.473 & 0.620 & 0.336 & 0.624 & 0.340 & 0.625 & 0.383 & 0.804 & 0.509 & 0.610 & 0.376 \\
        \midrule
        \multirow{5}{*}{Weather} & 96 & \textbf{0.155} & \textbf{0.197} & 0.166 & \underline{0.208} & 0.174 & 0.214 & 0.176 & 0.217 & 0.166 & 0.210 & \underline{0.158} & 0.230 & 0.202 & 0.261 & 0.172 & 0.220 & 0.173 & 0.223 & 0.196 & 0.255 & 0.221 & 0.306 & 0.217 & 0.296 \\
        & 192 & \textbf{0.206} & \textbf{0.245} & 0.217 & \underline{0.253} & 0.221 & 0.254 & 0.221 & 0.256 & 0.215 & 0.256 & \textbf{0.206} & 0.277 & 0.242 & 0.298 & 0.219 & 0.261 & 0.245 & 0.285 & 0.237 & 0.296 & 0.261 & 0.340 & 0.276 & 0.336 \\
        & 336 & \textbf{0.265} & \textbf{0.286} & 0.282 & 0.300 & 0.278 & \underline{0.296} & \underline{0.275} & \underline{0.296} & 0.287 & 0.300 & \underline{0.272} & 0.335 & 0.287 & 0.335 & 0.280 & 0.306 & 0.321 & 0.338 & 0.283 & 0.335 & 0.309 & 0.378 & 0.339 & 0.380 \\
        & 720 & \textbf{0.347} & \textbf{0.340} & 0.356 & 0.351 & 0.358 & 0.347 & \underline{0.352} & \underline{0.346} & 0.355 & 0.348 & 0.398 & 0.418 & \underline{0.351} & 0.386 & 0.365 & 0.359 & 0.414 & 0.410 & 0.345 & 0.381 & 0.377 & 0.427 & 0.403 & 0.428 \\
        \cmidrule{2-26}
        & Avg & \textbf{0.243} & \textbf{0.267} & \underline{0.255} & \underline{0.278} & 0.258 & \underline{0.278} & 0.256 & 0.279 & 0.256 & 0.279 & 0.259 & 0.315 & 0.271 & 0.320 & 0.259 & 0.287 & 0.288 & 0.314 & 0.265 & 0.317 & 0.292 & 0.363 & 0.309 & 0.360 \\
        \midrule
        \multirow{5}{*}{Solar} & 96 & \textbf{0.187} & \textbf{0.218} & \underline{0.200} & \underline{0.230} & 0.203 & 0.237 & 0.205 & 0.246 & 0.221 & 0.275 & 0.310 & 0.331 & 0.312 & 0.399 & 0.250 & 0.292 & 0.215 & 0.249 & 0.290 & 0.378 & 0.237 & 0.344 & 0.242 & 0.342 \\
        & 192 & \textbf{0.223} & \textbf{0.243} & \underline{0.229} & \underline{0.253} & 0.233 & 0.261 & 0.237 & 0.267 & 0.268 & 0.306 & 0.734 & 0.725 & 0.339 & 0.416 & 0.296 & 0.318 & 0.254 & 0.272 & 0.320 & 0.398 & 0.280 & 0.380 & 0.285 & 0.380 \\
        & 336 & \textbf{0.237} & \textbf{0.257} & \underline{0.243} & \underline{0.269} & 0.248 & 0.273 & 0.250 & 0.276 & 0.272 & 0.294 & 0.750 & 0.735 & 0.368 & 0.430 & 0.319 & 0.330 & 0.290 & 0.296 & 0.353 & 0.415 & 0.304 & 0.389 & 0.282 & 0.376 \\
        & 720 & \textbf{0.239} & \textbf{0.258} & \underline{0.245} & \underline{0.272} & 0.249 & 0.275 & 0.252 & 0.275 & 0.281 & 0.313 & 0.769 & 0.765 & 0.370 & 0.425 & 0.338 & 0.337 & 0.285 & 0.200 & 0.356 & 0.413 & 0.308 & 0.388 & 0.357 & 0.427 \\
        \cmidrule{2-26}
        & Avg & \textbf{0.221} & \textbf{0.244} & \underline{0.229} & \underline{0.256} & 0.233 & 0.262 & 0.236 & 0.266 & 0.260 & 0.297 & 0.641 & 0.639 & 0.347 & 0.417 & 0.301 & 0.319 & 0.261 & 0.381 & 0.330 & 0.401 & 0.282 & 0.375 & 0.291 & 0.381 \\
        \midrule
        \textbf{$1^{st} / 2^{nd}$} & count & 54 & 10 & 11 & 34 & 0 & 7 & 1 & 11 & 0 & 1 & 1 & 2 & 0 & 3 & 0 & 0 & 0 & 0 & 0 & 0 & 0 & 0 & 3 & 0 \\
        \bottomrule
    \end{tabular}
    \caption{Performance comparison across 8 real-world datasets with $L=96$ and varying $H$: CASA achieves the best results in 54 out of 64 metrics (84.4\%)}
    \label{tab:full results}
\end{table*}

\section{Proofs} \label{Apen:E}
\subsection{Proof of Proposition 1.}
\begin{proof}
For matrix $A$, define the $r$-th row as $\mathbf{row}_r(A)$ and the $s$-th column as $\mathbf{col}_s(A)$. Furthermore, for the vector $v$, define its $l$-th component as $v(l)$. Then, the feature for the $r$-th variate in $Q_i$, denoted as $\mathbf{row}_r(Q_i)$ is calculated as follows:
%Then, the feature for the $r$-th variate in $Q_i$, denoted as $\mathbf{row}_r(Q_i)$, and the $l$-th value of $\mathbf{row}_r(Q_i)$, denoted as $\mathbf{row}_r(Q_{i+1})(l)$, are calculated as follows:
\begin{align}
    %\mathbf{row}_r(Q_{i+1})(l) &= < \mathbf{row}_r(Z_{i,1}),\; \mathbf{col}_l(W_{i,1}) > \\
    %& \quad + \mathbf{row}_r (b_{i,1})(l) \\ 
    \mathbf{row}_r(Q_{i+1}) &= \mathbf{row}_r(Z_{i,1}) * W_{i,1} + \mathbf{row}_r (b_{i,1})
\end{align}
Take any $\Delta Z_{i,1} \in \mathbb{R}^{N \times D}$ such that $\; \mathbf{row}_r(\Delta Z_{i,1}) = \vec{0}$. Define $\widetilde{Z}_{i,1} \coloneqq Z_{i,1} + \Delta Z_{i,1}$. Then $\mathbf{row}_r(\widetilde{Z}_{i,1}) = \mathbf{row}_r(Z_{i,1})$. Therefore,
\begin{align}
    \mathbf{row}_r(Q_{i+1}) &= \mathbf{row}_r(Z_{i,1}) * W_{i,1} +\mathbf{row}_r (b_{i,1}) \\
                        &= \mathbf{row}_r(\widetilde{Z}_{i,1}) * W_{i,1} + \mathbf{row}_r (b_{i,1}) \\
                        &= \mathbf{row}_r(\widetilde{Q}_{i+1})
\end{align}
Thus, for any $r$, the value of $\mathbf{row}_r(Q_{i+1})$ depends solely on the columns of $W_{i,1}$ without interacting with other rows, that is, other variates. Similarly, key embeddings reach the same conclusion in the same way.
\end{proof}

\subsection{Proof of Proposition 2.}
\begin{proof}
By transposing $Z_{i,1}$, it can be proven in the same way as Proposition 1
\end{proof}

%\end{comment}

\end{document}